\documentclass{article}

\usepackage[preprint]{neurips_2025}

\usepackage[utf8]{inputenc} 
\usepackage[T1]{fontenc}    
\usepackage{hyperref}       
\usepackage{url}            
\usepackage{booktabs}       
\usepackage{amsfonts}       
\usepackage{nicefrac}       
\usepackage{microtype}      

\usepackage{graphicx}
\usepackage{algorithm}
\usepackage{algorithmic}
\usepackage{float}

\usepackage{amsmath}
\usepackage{amsthm}
\newtheorem{proposition}{Proposition}
\newtheorem{theorem}{Theorem}

\usepackage[table,dvipsnames]{xcolor}

\hypersetup{hidelinks, colorlinks=true, linkcolor=RoyalPurple, citecolor=RoyalBlue}

\title{CPGD: Toward Stable Rule-based Reinforcement Learning for Language Models}

\author{%
\begin{minipage}{\textwidth}
    \vspace{2em}
    \centering
    Zongkai Liu\textsuperscript{1,2}\thanks{Equal contribution} \quad Fanqing Meng\textsuperscript{4*} \quad Lingxiao Du\textsuperscript{3*} \quad Zhixiang Zhou\textsuperscript{2*} \\
    Chao Yu\textsuperscript{1$\dagger$} \quad Wenqi Shao\textsuperscript{2,3$\dagger$} \quad Qiaosheng Zhang\textsuperscript{2,3}\thanks{Corresponding Authors: \{zhangqiaosheng, shaowenqi\}@pjlab.org.cn; yuchao3@mail.sysu.edu.cn}
    \\[8pt]
    \textsuperscript{1}Sun Yat-Sen University \quad 
    \textsuperscript{2}Shanghai Innovation Institute \quad \\
    \textsuperscript{3}Shanghai AI Laboratory \quad
    \textsuperscript{4}Shanghai Jiao Tong University
\end{minipage}
}

\begin{document}


\maketitle

\begin{abstract}
  Recent advances in rule-based reinforcement learning (RL) have significantly improved the reasoning capability of language models (LMs) with rule-based rewards. However, existing RL methods---such as GRPO, REINFORCE++, and RLOO---often suffer from training instability, where large policy updates and improper clipping can lead to training collapse. To address this issue, we propose Clipped Policy Gradient Optimization with Policy Drift (CPGD), a novel algorithm designed to stabilize policy learning in LMs. CPGD introduces a policy drift constraint based on KL divergence to dynamically regularize policy updates, and leverages a clip mechanism on the logarithm of the ratio to prevent excessive policy updates. We provide theoretical justification for CPGD and demonstrate through empirical analysis that it mitigates the instability observed in prior approaches. Furthermore, we show that CPGD significantly improves performance while maintaining training stability. Our implementation balances theoretical rigor with practical usability, offering a robust alternative for RL in the post-training of LMs. We release our code at~\url{https://github.com/ModalMinds/MM-EUREKA}. 
\end{abstract}

\section{Introduction}\label{sec: introduction}


Rule-based reinforcement learning (RL) has emerged as a key approach for eliciting reasoning capabilities in language models (LMs)~\citep{deepseekai2025}. It leverages simple, efficient reward functions derived from deterministic rules, effectively mitigating reward hacking~\citep{gao2022scalinglawsrewardmodel} while activating reasoning abilities of models~\citep{deepseekai2025,polu2020generative, le2022coderl,shinn2023reflexion}. This has sparked a line of research focused on developing more effective RL algorithms for both textual and general multimodal reasoning tasks. Notable methods include GRPO~\citep{deepseekai2025}, REINFORCE++~\citep{hu2025reinforce++}, RLOO~\citep{RLOOKoolHW19a, RLOOAhmadianCGFKPUH24}, and GRPO variants such as DAPO~\citep{yu2025dapoopensourcellmreinforcement}, Dr.GRPO~\citep{liu2025drgrpounderstandingr1zeroliketrainingcritical}, and GPG~\citep{chu2025gpgsimplestrongreinforcement}. However, we observe that these RL methods often suffer from training instability, which we attribute to the use of \textit{importance-sampling ratios} in their loss functions. Although PPO-clip loss~\citep{PPO} is commonly adopted to mitigate extreme policy updates, its one-sided nature fails to constrain large ratios when the advantage is negative---potentially causing gradient explosions dominated by poor samples, leading to catastrophic training collapse. We theoretically show that incorporating the importance-sampling ratio in the loss can amplify the policy shift, and our empirical results confirm that this can lead to training collapse in existing RL methods.

To address this issue, we propose \textit{Clipped Policy Gradient Optimization with Policy Drift} (CPGD), an algorithm that replaces the PPO-clip loss with a policy gradient loss~\citep{sutton1998reinforcement} to avoid instability caused by directly involving policy ratios in the loss function. To ensure proximal optimization, we introduce both a clip mechanism and a policy drift regularizer, constraining optimization within a local region and mitigating over-optimization that may impair reasoning behaviors as shown in Section~\ref{sec: training collapse}. Furthermore, we develop a novel KL estimator that ensures correct gradient directions while avoiding the potential numerical instability associated with the commonly used $k_3$ estimators~\citep{schulman2023approximating}. We also incorporate weighted advantages to dynamically adjust the influence of each sample, further enhancing model performance.

We theoretically prove the convergence of CPGD and empirically demonstrate its superior training stability and performance. As shown in Table~\ref{tab:benchmark_comparison}, models trained with CPGD consistently outperform those trained with other RL algorithms and strong open-source baselines across standard multimodal reasoning benchmarks. Notably, CPGD improves the overall performance over the base model by +11.0\% across all benchmarks. Specially, CPGD achieves +21.8\% gain on the in-domain benchmark MMK12~\citep{meng2025mmeurekaexploringfrontiersmultimodal}, and improves by +8.5\% and +11.4\% on the out-of-distribution benchmarks MathVista~\citep{lu2024mathvistaevaluatingmathematicalreasoning} and MathVision~\citep{wang2024measuringmultimodalmathematicalreasoning}, respectively.

\section{Related work}

\textbf{RL for training reasoning models.}
RL has become a key method for improving reasoning in LMs~\citep{deepseekai2025, openai2024o1}. While early methods rely on PPO~\citep{PPO}, its high computational cost has driven interest in alternatives like DPO~\citep{rafailov2023direct}, which simplifies training but depends on offline data. Recent RL methods such as GRPO, RLOO, and REINFORCE++ aim to balance stability and efficiency. Notably, DeepSeek R1~\citep{deepseekai2025} shows that pure RL can elicit self-reflection and reasoning in LMs without supervised pretraining. 
Recently, several concurrent works have proposed GRPO variants to address its limitations. For instance, Dr.GRPO~\citep{liu2025drgrpounderstandingr1zeroliketrainingcritical} identifies optimization bias in GRPO that favors longer response among incorrect ones. DAPO~\citep{yu2025dapoopensourcellmreinforcement} incorporates multiple improvements, including decoupled clipping thresholds, token-level losses, and an online filtering strategy. GPG~\citep{chu2025gpgsimplestrongreinforcement}, in contrast, adopts a minimalist design by discarding both clipping and KL regularization, relying solely on the policy gradient loss~\citep{sutton1998reinforcement}. 
However, none of these approaches fundamentally resolve the training instability issue to existing RL methods, which is the primary focus of this work.

\textbf{Large reasoning model.}
Recently, a surge of reasoning models has emerged, driven by the principle of test-time scaling laws, which demonstrate that models with explicit reasoning processes achieve superior performance~\citep{chen2025expandingperformanceboundariesopensource}. Leading models in this area include DeepSeek R1~\citep{deepseekai2025}, OpenAI's o-series~\citep{openai2024o1}, Qwen series~\citep{qwq32b, qvq-72b-preview}, and Kimi k1.5~\citep{kimiteam2025kimik15scalingreinforcement}. However, their training pipelines and datasets remain undisclosed.
This has motivated a wave of academic research within the open-source community, including parallel efforts such as OpenR1~\citep{openr1}, TinyZero~\citep{tinyzero}, LMM-R1~\citep{peng2025lmmr1}, R1-V~\citep{chen2025r1v}, Reason-RFT~\citep{tan2025reasonrftreinforcementfinetuningvisual}, and MM-Eureka~\citep{meng2025mmeurekaexploringfrontiersmultimodal}. These works primarily focus on constructing high-quality datasets and complete training pipelines. 
They commonly adopt GRPO to enhance reasoning capabilities but do not specifically investigate improvements to the RL algorithms themselves. 


\section{Preliminaries}

\subsection{Problem formulation}

We denote an LM by $\pi_{\theta}$, where $\theta \in \mathbb{R}^d$ represents the model parameters. Given a prompt $\mathbf{x} = [x_1, \dots, x_m] \in \mathcal{D}$, the model generates a response $\mathbf{y} = [y_1, \dots, y_n]$ by sampling from the conditional distribution $\pi_\theta(\cdot|\mathbf{x})$, with both $x_i$ and $y_i$ drawn from a finite vocabulary $\mathcal{V}$. In this work, we focus on transformer-based LMs that generate responses autoregressively, such that \(\pi_\theta(\mathbf{y}|\mathbf{x}) = \prod_{i=1}^{n}\pi_\theta(y_i|\mathbf{x}, \mathbf{y}_{<i})\), 
where $\mathbf{y}_{<i} = [y_1,\dots,y_{i-1}]$ and $\mathbf{y}_{<1}$ is an empty sequence. 

RL in post-training is typically modeled as a Markov decision process (MDP), defined by a tuple $\mathcal{M} = (\mathcal{S}, \mathcal{A}, \mathcal{P}, \mathcal{R}, \rho)$, where $\mathcal{S}$ is the state space, $\mathcal{A}$ is the action space, $\mathcal{P}$ is the transition kernel, $\mathcal{R}$ is the deterministic reward function, and $\rho$ is the initial state distribution. For LMs, two MDP formulations are widely considered: \textit{token-level MDP} and \textit{response-level MDP}. 

In a \textit{token-level MDP}, each token is treated as a single action. At the time step $t$, the state $\mathbf{s}_t = [\mathbf{x}, \mathbf{y}_{<t}]$ includes the prompt and the tokens generated so far. The action $a_t = y_t$ is sampled according to $y_t \sim \pi_\theta(\cdot | \mathbf{x}, \mathbf{y}_{<t})$, where the action space $\mathcal{A}$ is equal to the vocabulary $\mathcal{V}$. The environment transitions deterministically to $\mathbf{s}_{t+1} = [\mathbf{x}, \mathbf{y}_{<t+1}]$. The reward is defined as $\mathcal{R}(\mathbf{s}_t, a_t) = \mathcal{R}([\mathbf{x}, \mathbf{y}_{<t}], y_t)$, and $\rho$ is induced by the prompt distribution in $\mathcal{D}$.

In a \textit{response-level MDP}, the full response is treated as an individual action: $\mathbf{a} = \mathbf{y} \sim \pi_\theta(\cdot | \mathbf{x})$. The state is defined solely by the prompt $\mathbf{s} = \mathbf{x}$, and the episode terminates after one step. Thus, the transition kernel is omitted in the single-turn dialogue setting. The reward is $\mathcal{R}(\mathbf{s}, \mathbf{a}) = \mathcal{R}(\mathbf{x}, \mathbf{y})$, with $\rho$ again determined by $\mathcal{D}$.

\subsection{Rule-based reinforcement learning}

This work focuses on verifiable tasks, where the outcome reward is determined by the final accuracy. Specifically, a response $\mathbf{y}$ receives a reward of 1 if it is the correct answer to the prompt $\mathbf{x}$, and 0 otherwise. We denote this reward function as $\mathcal{R}_o$ to emphasize its nature as an outcome-based reward. 
Within this setting, REINFORCE-style algorithms are favored as they reduce computational cost by forgoing critic networks. Notable methods include REINFORCE++~\citep{hu2025reinforce++}, RLOO~\citep{RLOOKoolHW19a, RLOOAhmadianCGFKPUH24}, and GRPO~\citep{deepseekai2025}. 

\textbf{REINFORCE++: }
REINFORCE++ enhances the standard REINFORCE framework by integrating key optimizations from PPO~\citep{PPO}, improving both stability and efficiency. The objective is defined as:
\begin{align*}
    \mathcal{L}_\text{R++}(\theta; \theta_{old}) = 
    \mathbb{E}_{\mathbf{x}\sim\mathcal{D}, \mathbf{y}\sim\pi_{\theta_{old}}(\cdot|\mathbf{x})}
    \Bigg[
    \frac{1}{|\mathbf{y}|}&\sum_{i=1}^{|\mathbf{y}|} 
    \min\Bigg(  \frac{\pi_{\theta}(y_i|\mathbf{x}, \mathbf{y}_{<i})}{\pi_{\theta_{old}}(y_i|\mathbf{x}, \mathbf{y}_{<i})}  A_i^\text{R++}, 
    \\&\qquad \operatorname{clip}_{1-\epsilon}^{1+\epsilon}\Big(  \frac{\pi_{\theta}(y_i|\mathbf{x}, \mathbf{y}_{<i})}{\pi_{\theta_{old}}(y_i|\mathbf{x}, \mathbf{y}_{<i})}\Big)A_i^\text{R++} \Bigg)
    \Bigg],
\end{align*}
where $
\epsilon\in[0, 1]$, $\operatorname{clip}_{a}^{b}(x):= \max(\min(x, b), a)$, and
\[
    A_i^\text{R++} := \operatorname{GlobalNorm}\Big(G(\mathbf{x}, \mathbf{y}_{\le i})\Big), \quad 
    G(\mathbf{x}, \mathbf{y}_{\le i}) := 
    \mathcal{R}_o(\mathbf{x}, \mathbf{y}) - \beta \sum_{j=i}^{|\mathbf{y}|} \ln\frac{\pi_{\theta_{old}}(y_j|\mathbf{x}, \mathbf{y}_{<j})}{\pi_\text{ref}(y_j|\mathbf{x}, \mathbf{y}_{<j})}. 
\]
Here, $\ln\frac{\pi_{\theta_{old}}(y_j|\mathbf{x}, \mathbf{y}_{<j})}{\pi_\text{ref}(y_j|\mathbf{x}, \mathbf{y}_{<j})}$ is the token-level KL penalty, constraining divergence from the reference policy $\pi_\text{ref}$, typically the initial model. $\operatorname{GlobalNorm}(x) = \frac{x-\operatorname{mean}(\{x^\prime\in\text{ batch}\})}{\operatorname{std}(\{x^\prime\in\text{ batch}\})}$ is the normalization operation across the global batch for all prompts.

\textbf{RLOO: }The primary distinction between RLOO and REINFORCE++ lies in their computation of the advantage value. RLOO first generates a group of $K$ responses $\{\mathbf{y}^{(k)}\}_{k=1}^K$ for each prompt $\mathbf{x}$ and computes the advantage using a \textit{leave-one-out} strategy to reduce the gradient variance: 
\[
    A_{i, k}^\text{RLOO} := \operatorname{GlobalNorm}\Big(\tilde G(\mathbf{x}, \mathbf{y}^{(k)}_{\le i})\Big),
    \quad 
    \tilde G(\mathbf{x}, \mathbf{y}^{(k)}_{\le i}) := G(\mathbf{x}, \mathbf{y}^{(k)}_{\le i})
    -\frac{1}{K-1}\sum_{k^\prime\neq k}G(\mathbf{x}, \mathbf{y}^{(k^\prime)}_{\le i}). 
\]

\textbf{GRPO: }
GRPO introduces a group-based advantage and employs an external KL regularization via the $k_3$ estimator~\citep{schulman2023approximating}, which approximates $D_\text{KL}(p,q) = \sum_{i}(q_i/p_i-1-\ln q_i/p_i)$. The loss is: 
\begin{align*}
    \mathcal{L}_\text{GRPO}(\theta; \theta_{old}) &= 
    \mathbb{E}_{\mathbf{x}\sim\mathcal{D}, \{\mathbf{y}^{(k)}\}^K_{k=1}\sim\pi_{\theta_{old}}(\cdot|\mathbf{x})}
    \Bigg[\frac{1}{K}\sum_{k=1}^K\Bigg(
    \frac{1}{|\mathbf{y}^{(k)}|}\sum_{i=1}^{|\mathbf{y}^{(k)}|} 
    \Bigg(-\beta \cdot \mathcal{M}_{\theta, \text{ref}}^{i}(\mathbf{x}, \mathbf{y}^{(k)})
    \\&+
    \min\Big( \frac{\pi_{\theta}(y_i^{(k)}|\mathbf{x}, \mathbf{y}^{(k)}_{<i})}{\pi_{\theta_{old}}(y^{(k)}_i|\mathbf{x}, \mathbf{y}^{(k)}_{<i})} A_{k}^\text{GRPO}, \operatorname{clip}_{1-\epsilon}^{1+\epsilon}\Big(\frac{\pi_{\theta}(y_i^{(k)}|\mathbf{x}, \mathbf{y}^{(k)}_{<i})}{\pi_{\theta_{old}}(y^{(k)}_i|\mathbf{x}, \mathbf{y}^{(k)}_{<i})}\Big)A_{k}^\text{GRPO} \Big)
    \Bigg)\Bigg],
\end{align*}
where
\begin{align*}
    &A_{k}^\text{GRPO} := \operatorname{GroupNorm}(\mathcal{R}_o(\mathbf{x}, \mathbf{y}^{(k)}))=\frac{\mathcal{R}_o(\mathbf{x}, \mathbf{y}^{(k)}) - \operatorname{mean}(\{\mathcal{R}_o(\mathbf{x}, \mathbf{y}^{(k)})\}_{k=1}^K)}{\operatorname{std}(\{\mathcal{R}_o(\mathbf{x}, \mathbf{y}^{(k)})\}_{k=1}^K)},
    \\
    &\mathcal{M}_{\theta, \text{ref}}^{i}(\mathbf{x}, \mathbf{y}^{(k)}) := 
    \frac{\pi_\text{ref}(y_i^{(k)}|\mathbf{x}, \mathbf{y}^{(k)}_{<i})}{\pi_{\theta}(y_i^{(k)}|\mathbf{x}, \mathbf{y}^{(k)}_{<i})} - 1 - \ln \frac{\pi_\text{ref}(y_i^{(k)}|\mathbf{x}, \mathbf{y}^{(k)}_{<i})}{\pi_{\theta}(y^{(k)}_i|\mathbf{x}, \mathbf{y}^{(k)}_{<i})}. 
\end{align*}

\section{The proposed method}

This section introduces our RL algorithm, \textit{Clipped Policy Gradient Optimization with Policy Drift} (CPGD), designed to improve the stability of RL training. In Section~\ref{sec: CPGD}, we present the CPGD algorithm along with its theoretical guarantees, and highlight potential limitations of the standard PPO-clip loss. In Section~\ref{sec: training collapse}, we provide empirical evidence of instability in existing methods and analyze its possible causes, showing how CPGD addresses them for more stable training. Finally, Section~\ref{sec: implementation} describes the practical implementation of CPGD, striking a balance between theoretical soundness and practical implementation.

\subsection{Clipped Policy Gradient Optimization with Policy Drift (CPGD)}\label{sec: CPGD}

Under the response-level MDP assumption, CPGD aims to maximize the following formula:
\begin{equation}\label{eq: CPGD theory}
    \mathcal L_\text{CPGD}(\theta; \theta_{old})=\mathbb{E}_{\mathbf{x}\sim\mathcal D}\Big[\mathbb{E}_{\mathbf{y}\sim\pi_{\theta_{old}}(\cdot|\mathbf{x})}
    \big[\Phi_{\theta}(\mathbf{x}, \mathbf{y})\big]
    -\alpha \cdot {D}_\text{KL}(\pi_{\theta_{old}}, \pi_{\theta}|\mathbf{x})
    \Big],
\end{equation}
where
\begin{align*}
    &\Phi_{\theta}(\mathbf{x}, \mathbf{y}) := 
    \min\Big( \ln\frac{\pi_\theta(\mathbf{y}|\mathbf{x})}{\pi_{\theta_{old}}(\mathbf{y}|\mathbf{x})} \cdot A^\text{CPGD}(\mathbf{x}, \mathbf{y}), \operatorname{clip}_{\ln(1-\epsilon)}^{\ln(1-\epsilon)}\Big(\ln\frac{\pi_\theta(\mathbf{y}|\mathbf{x})}{\pi_{\theta_{old}}(\mathbf{y}|\mathbf{x})}\Big)A^\text{CPGD}(\mathbf{x}, \mathbf{y}) \Big),
    \\
    &A^\text{CPGD}(\mathbf{x}, \mathbf{y}) := \mathcal{R}_o(\mathbf{x}, \mathbf{y}) - \mathbb{E}_{\mathbf{y}^\prime\sim\pi_\theta(\cdot|\mathbf{x})}\big[ \mathcal{R}_o(\mathbf{x}, \mathbf{y}^\prime) \big],
    \\
    &{D}_\text{KL}(\pi_{\tilde\theta}, \pi_{\theta}|\mathbf{x}) := \mathbb{E}_{\mathbf{y}\sim\pi_{\tilde\theta}(\cdot|\mathbf{x})} \Big[ \ln\frac{\pi_{\tilde\theta}(\mathbf{y}|\mathbf{x})}{\pi_{\theta}(\mathbf{y}|\mathbf{x})} \Big].
\end{align*}
Hereinafter, we term the KL divergence between the old and current policies as \textit{policy drift}, and between the current and reference policies as \textit{reference constraint}.  

CPGD differs from the standard PPO-clip loss in two key aspects: 
(1) A different policy optimization objective is used, where the policy gradient loss is adopted with the clip mechanism. 
(2) A policy drift is introduced, imposing a forward KL divergence penalty between the old and current policies.

\textbf{Why use the policy gradient objective?}
In the original PPO objective, although the importance-sampling ratio corrects for the distribution mismatch between the old and current policies, it simultaneously introduces high variance. As empirically demonstrated in Section~\ref{sec: training collapse}, such variance can destabilize training and even cause training collapse, while using a policy gradient loss without the ratio substantially improves training stability. Proposition~\ref{prop: ratio out of the range} further provides a theoretical explanation for this phenomenon, showing that the use of the policy ratio amplifies policy drift, causing the updated policy to exceed the intended bounds. 

\textbf{Why introduce the policy drift and clip mechanism?}
The introduction of the clip mechanism and policy drift is designed to ensure proximal policy updates, which are critical for theoretical convergence guarantees in Theorem~\ref{thrm: convergence of CPGD}. The clip mechanism enforces local updates by zeroing gradients when the policy ratio exceeds a specified threshold, while policy drift introduces a corrective gradient to constrain the policy ratio within a stable range. Notably, the clip mechanism alleviates the need for a large penalty coefficient on the policy drift term: when the ratio remains within bounds, the small drift coefficient allows the algorithm to focus on optimizing the primary objective $\Phi$; when the ratio exceeds the range, the gradient of the primary objective becomes zero, prompting the algorithm to rely on the policy drift signal to prevent further deviation---particularly those caused by optimizer momentum (e.g., Adam) or neural network generalization effects.

\begin{proposition}\label{prop: ratio out of the range}
    Let $\theta_0$ be a parameter such that the importance-sampling ratio satisfies $|\frac{\pi_{\theta_0}(\mathbf{y}|\mathbf{x})}{\pi_{\theta_{old}}(\mathbf{y}|\mathbf{x})} - 1|= \epsilon$. Consider updating $\theta_0$ using either (i) the PPO-clip objective, resulting in parameter $\theta_1^{\text{PPO}}$, or (ii) the CPGD objective with $\alpha = 0$ (denoted as CPG), yielding parameter $\theta_1^{\text{CPG}}$. Then, there exists a constant $\eta_{\max} > 0$ such that for any learning rate $\eta \in (0, \eta_{\max})$, the following inequality holds: 
    \[
        \Bigg|\frac{\pi_{\theta_1^\text{PPO}}(\mathbf{y}|\mathbf{x})}{\pi_{\theta_{old}}(\mathbf{y}|\mathbf{x})} - 1\Bigg|>
        \Bigg|\frac{\pi_{\theta_1^\text{CPG}}(\mathbf{y}|\mathbf{x})}{\pi_{\theta_{old}}(\mathbf{y}|\mathbf{x})} - 1\Bigg|>
        \Bigg|\frac{\pi_{\theta_0}(\mathbf{y}|\mathbf{x})}{\pi_{\theta_{old}}(\mathbf{y}|\mathbf{x})} - 1\Bigg|= \epsilon.
    \]
    After one update step, both PPO and CPG increase the importance-sampling ratio deviation from the old policy, but PPO does so more aggressively than CPG. 
\end{proposition}

The following theorem further presents that CPGD enjoys the convergence guarantee, indicating its theoretical rationality. See Appendix~\ref{apx: proof} for the proofs of Proposition~\ref{prop: ratio out of the range} and Theorem~\ref{thrm: convergence of CPGD}. 
\begin{theorem}\label{thrm: convergence of CPGD}
    Let $\{\pi_{\theta_k}\}_{k=0}^\infty$ denote the sequence of policies generated by the CPGD update rule (Equation~\ref{eq: CPGD theory}). Then, the sequence ${\pi_{\theta_k}}$ converges. 
\end{theorem}

\subsection{Training collapse}\label{sec: training collapse}


\begin{figure}[t]
    \centering
    \includegraphics[width=0.8\linewidth]{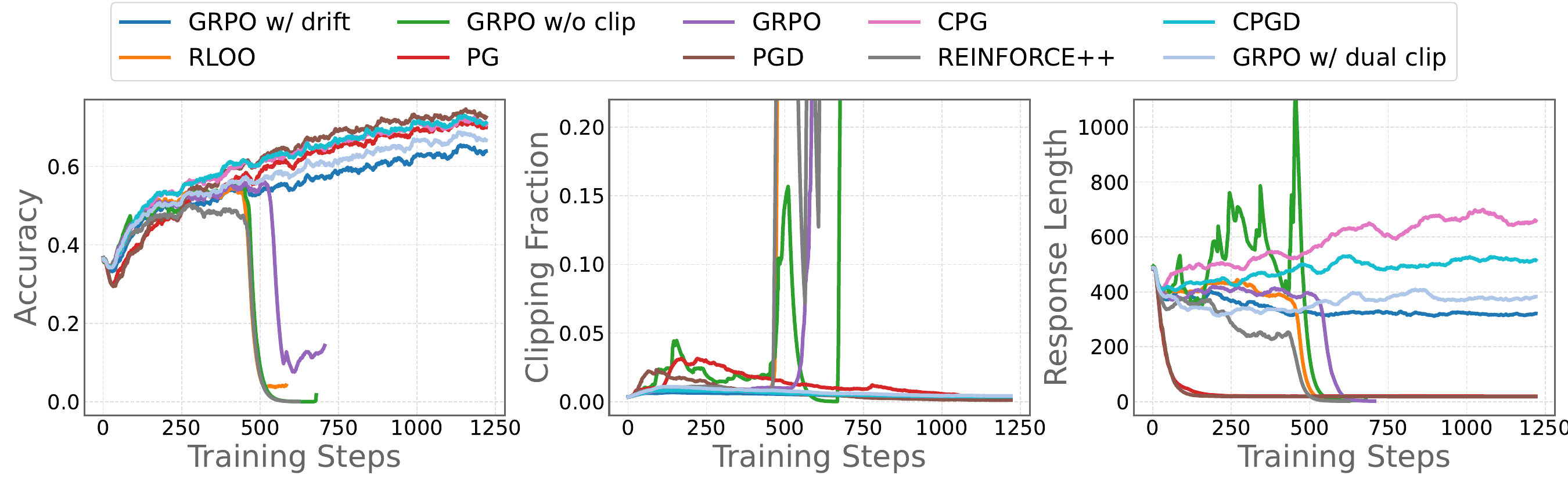}
    \caption{{Accuracy, clipping fraction and response length curves throughout training.} }
    \label{fig: training collapse}
\end{figure}

Several studies suggest that the reference constraint may hinder policy improvement~\citep{yu2025dapoopensourcellmreinforcement, OpenReasonerZero2025}. However, we observe that removing this KL term leaves the PPO-clip loss alone insufficient to effectively constrain large policy shifts, which can lead to training collapse. While such collapse may be partially mitigated through techniques such as early stopping or small learning rates, it remains a latent instability that undermines the reliability of continued training. In this subsection, we investigate this phenomenon of training collapse and demonstrate that CPGD effectively prevents it.

Figure~\ref{fig: training collapse} presents training curves on the MMK12 dataset~\citep{meng2025mmeurekaexploringfrontiersmultimodal} for RLOO, REINFORCE++, GRPO, GRPO w/o clip (i.e., GRPO without the clip mechanism), GRPO w/ dual clip (i.e., the policy ratio is additionally clipped to no more than a constant---$3.0$ in our case---when advantage is negative~\citep{ye2020mastering}), GRPO w/ drift (i.e., GRPO with policy drift), PG (basic policy gradient), CPG (PG with the clip mechanism), PGD (PG with the policy drift), and CPGD, all without the reference constraint. We use QwenVL2.5-7B~\citep{bai2023qwenvlversatilevisionlanguagemodel} as the base model. All algorithms share the same hyperparameters: a training and rollout batch size of $128$, $8$ responses per prompt, a learning rate of $1e{-}6$, one PPO epoch, and ten training episodes. As shown in Figure~\ref{fig: training collapse}, almost all baselines experience training collapse.

As shown in Figure~\ref{fig: training collapse}, methods such as REINFORCE++, RLOO, GRPO w/o clip, and GRPO exhibit highly unstable policy ratio dynamics, leading to training collapse in mid stages. In contrast, GRPO w/ dual clip, GRPO w/ drift, PG, CPG, PGD, and CPGD maintain stable training curves. GRPO w/ dual clip mitigates instability by globally constraining the policy ratio, while the PG series sidesteps ratio-induced variance by excluding it from the loss computation. These comparisons indicate that incorporating policy ratios in the loss can introduce high variance during fluctuations, and that simple one-sided clipping fails to recover from extreme ratios, ultimately causing collapse. Although dual clip mechanism stabilizes training, it may introduce new issues: frequent zero-gradient updates and ineffective learning under negative advantages due to the zero-gradient clipped large ratios. Additionally, GRPO w/ drift demonstrates that incorporating policy drift effectively constrains the policy ratio within a reasonable range, thereby preventing training collapse.

On the other hand, while prior work suggests clipping may be unnecessary due to the low proportion of clipped ratios~\citep{RLOOAhmadianCGFKPUH24, chu2025gpgsimplestrongreinforcement}, our findings suggest otherwise. Despite only \textasciitilde1\% of ratios being clipped, training performance diverges significantly with and without clipping. Specifically, methods like PG and PGD—though stable without ratio terms—suffer from response length collapse, degenerating into trivial outputs (e.g., only emitting tokens like \textless think\textgreater) that exploit the format reward function without performing meaningful reasoning. This highlights the model’s vulnerability to reward hacking, likely due to overly aggressive updates. These results reveal the necessity of the proximal policy updates.

\subsection{Implementation}\label{sec: implementation}

In this subsection, we design a practically implementable loss function in per-token form based on the CPGD update formulation (Equation~\ref{eq: CPGD theory}), aiming to strike a balance between theoretical rigor and empirical applicability. Our CPGD loss is straightforward to be integrated into widely-used large model training frameworks such as OpenRLHF~\citep{hu2024openrlhf} and veRL~\citep{sheng2024hybridflow}. The practical loss function is given by
\begin{equation}\label{eq: CPGD loss}
    \mathcal{J}_\text{CPGD}(\theta) = -\frac{1}{|\mathcal{D}|}\sum_{(\mathbf{x}, \{\mathbf{y}^{(k)}\}_{k=1}^{K})\in\mathcal D}
    \frac{1}{\sum_{k=1}^{K}|\mathbf{y}^{(k)}|}\Bigg[
    \sum_{i=1}^{|\mathbf{y}^{(k)}|}\Bigg(
    \Phi_{\theta}^{i}(\mathbf{x}, \mathbf{y}^{(k)})
    -\alpha \cdot \mathcal{E}^{i}_{\theta_{old}, \theta}(\mathbf{x}, \mathbf{y}^{(k)})
    \Bigg)
    \Bigg],
\end{equation}
where the per-token policy optimization term is
\begin{align*}
    &\Phi_{\theta}^{i}(\mathbf{x}, \mathbf{y}) \!\!:=\! 
    \min\!\!\Bigg(\!\! \ln\!\frac{\pi_\theta(y_i|\mathbf{x}, \mathbf{y}_{<i})}{\pi_{\theta_{old}}(y_i|\mathbf{x}, \mathbf{y}_{<i})}\!\cdot\! A_{\omega}^\text{CPGD}(\mathbf{x}, \mathbf{y}), 
    \operatorname{clip}_{\ln{(1-\epsilon_i)}}^{\ln{(1+\epsilon_i)}}\!\!\Big(\!\ln\frac{\pi_\theta(y_i|\mathbf{x}, \mathbf{y}_{<i})}{\pi_{\theta_{old}}(y_i|\mathbf{x}, \mathbf{y}_{<i})}\!\Big)A_{\omega}^\text{CPGD}(\mathbf{x}, \mathbf{y}) \!\!\Bigg),
\end{align*}
and
\begin{align*}
    &A^\text{CPGD}_{\omega}(\mathbf{x}, \mathbf{y}^{(k)}) := \omega(\mathbf{x})\cdot\Big({\mathcal{R}_o(\mathbf{x}, \mathbf{y}^{(k)}) - \operatorname{mean}\big( \{\mathcal{R}_o(\mathbf{x}, \mathbf{y}^{(k^\prime)})\}_{k^\prime=1}^K \big)}\Big),
    \\
    &\mathcal{E}^{i}_{{\theta_{old}}, \theta}(\mathbf{x}, \mathbf{y}) := 
    \min\Big(\frac{\operatorname{sg}(\pi_{\theta}(y_i|\mathbf{x}, \mathbf{y}_{<i}))}{\pi_{\theta_{old}}(y_i|\mathbf{x}, \mathbf{y}_{<i})} - 1, c \Big) \cdot\ln\pi_{\theta}(y_i|\mathbf{x}, \mathbf{y}_{<i}). 
\end{align*}
Here, $\operatorname{sg}(\cdot)$ denotes the operation that prevents gradient computation, $\omega(\mathbf{x})$ is a per-prompt weighting factor, and $c>0$ is a constant. 
We provide the following clarifications regarding the differences between the theoretical update formulation (Equation~\ref{eq: CPGD theory}) and the practical loss (Equation~\ref{eq: CPGD loss}):

\textbf{(I) Policy optimization term:}
In the theoretical update (Equation~\ref{eq: CPGD theory}), the policy optimization term is written in the form of joint distribution. But in the practical implementation (Equation~\ref{eq: CPGD loss}), it is decomposed into token level using the decomposability of the logarithm function. 
Specifically, the clipping threshold $\epsilon_i$ can be set the same for all tokens, ensuring that each token shares the same clip range. Alternatively, a tight-to-loose schedule can be employed such as $\epsilon_i =  \lambda \epsilon + (1-\lambda) \epsilon \cdot i / |\mathbf{y}^{(k)}|$, which assigns smaller thresholds to earlier tokens that usually have higher variance.

\textbf{(II) Policy drift:}
Similar to the policy optimization term, policy drift also leverages the decomposability of the logarithm function, but applies the following further transformations:
\begin{align}
    {D}_\text{KL}(\pi_{\theta_{old}}, \pi_{\theta}|\mathbf{x}) &= \mathbb{E}_{\mathbf{y}\sim\pi_{\theta_{old}}(\cdot|\mathbf{x})} \Big[ \ln\frac{\pi_{\theta_{old}}(\mathbf{y}|\mathbf{x})}{\pi_{\theta}(\mathbf{y}|\mathbf{x})} \Big]
    =
    \mathbb{E}_{\mathbf{y}\sim\pi_{\theta_{old}}(\cdot|\mathbf{x})} \Big[ \sum_{i=1}^{|\mathbf{y}|}\ln\frac{\pi_{\theta_{old}}(y_i|\mathbf{x}, \mathbf{y}_{<i})}{\pi_{\theta}(y_i|\mathbf{x}, \mathbf{y}_{<i})} \Big]
    \label{eq: k1}\\&=
    \mathbb{E}_{\mathbf{y}\sim\pi_{\theta_{old}}(\cdot|\mathbf{x})} \Big[ \sum_{i=1}^{|\mathbf{y}|}\Big(
    \frac{\pi_{\theta}(y_i|\mathbf{x}, \mathbf{y}_{<i})}{\pi_{\theta_{old}}(y_i|\mathbf{x}, \mathbf{y}_{<i})} - 1 - 
    \ln\frac{\pi_{\theta}(y_i|\mathbf{x}, \mathbf{y}_{<i})}{\pi_{\theta_{old}}(y_i|\mathbf{x}, \mathbf{y}_{<i})} \Big) \Big].
    \label{eq: k3}
\end{align}
Equations~\ref{eq: k1} and \ref{eq: k3} correspond to the $k_1$ and $k_3$ KL estimators proposed by Schulman~\citep{schulman2023approximating}. In practice, particularly when using gradient optimizers such as Adam, we prefer the $k_3$ estimator over $k_1$, as $k_1$ fails to effectively constrain the policy drift, while the gradient direction of $k_3$ dynamically adjusts based on the relative magnitude between the current and old policies:
\begin{align}
    &\nabla_{\theta}\ln\frac{\pi_{\theta_{old}}(y_i|\mathbf{x}, \mathbf{y}_{<i})}{\pi_{\theta}(y_i|\mathbf{x}, \mathbf{y}_{<i})} = -\nabla_{\theta}\ln{\pi_{\theta}(y_i|\mathbf{x}, \mathbf{y}_{<i})},
    \nonumber\\
    &\nabla_{\theta}\Big(\frac{\pi_{\theta}(y_i|\mathbf{x}, \mathbf{y}_{<i})}{\pi_{\theta_{old}}(y_i|\mathbf{x}, \mathbf{y}_{<i})} - 1 - 
    \ln\frac{\pi_{\theta}(y_i|\mathbf{x}, \mathbf{y}_{<i})}{\pi_{\theta_{old}}(y_i|\mathbf{x}, \mathbf{y}_{<i})}\Big)
    =\Big(\frac{\pi_{\theta}(y_i|\mathbf{x}, \mathbf{y}_{<i})}{\pi_{\theta_{old}}(y_i|\mathbf{x}, \mathbf{y}_{<i})} - 1\Big)
    \nabla_{\theta}\ln{\pi_{\theta}(y_i|\mathbf{x}, \mathbf{y}_{<i})}.\label{eq: k3 gradient}
\end{align}

However, Equation~\ref{eq: k3} involves the policy ratio, which can potentially lead to training collapse as discussed in Section~\ref{sec: training collapse}. To mitigate this issue, we clip the policy ratio to be no greater than $c+1$. Importantly, this clipping is not applied directly to the KL divergence estimator in Equation~\ref{eq: k3}, but rather to its gradient (Equation~\ref{eq: k3 gradient}). 
This design ensures that when the ratio exceeds the threshold, the policy drift term continues to provide a gradient that reduces the ratio: when $\frac{\operatorname{sg}(\pi_{\theta}(y_i|\mathbf{x}, \mathbf{y}_{<i}))}{\pi_{\theta_{old}}(y_i|\mathbf{x}, \mathbf{y}_{<i})}-1>c$, 
\[
    \nabla_{\theta} \mathcal{E}^{i}_{{\theta_{old}}, \theta}(\mathbf{x}, \mathbf{y}) = 
    c\cdot \nabla_{\theta}\ln\pi_{\theta}(y_i|\mathbf{x}, \mathbf{y}_{<i}). 
\]
In contrast, if clipping were applied to the estimator itself, the resulting gradient $-\nabla_{\theta}\ln\pi_{\theta}(y_i|\mathbf{x}, \mathbf{y}_{<i})$ would further increase the ratio once it exceeds the threshold, exacerbating training instability.

\textbf{(III) Weighted advantage:}
In the view of the response level, each prompt can be viewed as a distinct task. Consequently, we can introduce a per-prompt weighting factor $\omega(\mathbf{x})$ to assign different levels of importance to different prompts. 
(1) \textit{Equal weight}: when $\omega(\mathbf{x}) = 1$, $A^{\text{CPGD}}_{\omega}$ reduces to the original unweighted form.
(2) \textit{STD weight}: when $\omega(\mathbf{x})=1/\operatorname{std}(\{\mathcal{R}(\mathbf{x}, \mathbf{y}^{(k)})\}_{k})$, $A^{\text{CPGD}}_{\omega}$ is the same as $A^{\text{GRPO}}$. 
(3) \textit{Clip-filter-like weight}: when $\omega(\mathbf{x})=\min(c_{\omega}, \frac{\#\{\mathbf{x}\in\mathcal{D}\}}{\#\{\mathbf{x}\in\mathcal{D}\mid\operatorname{std}(\{\mathcal{R}_o(\mathbf{x}, \mathbf{y}^{(k)})\}_{k})\neq 0\}})$, $c_{\omega}>0$, similar weighting strategies have also been explored in concurrent work~\citep{chu2025gpgsimplestrongreinforcement}, with an analogous effect to online filtering~\citep{cui2025processreinforcementimplicitrewards}, amplifying the gradient contribution of samples with non-zero advantage. 

\section{Experiments}\label{sec: experiments}

\subsection{Experiments setup}

\textbf{RL baselines, dataset and implementation details. }
We compare our CPGD with several widely used RL algorithms, including GRPO~\citep{deepseekai2025}, REINFORCE++~\citep{hu2025reinforce++} and RLOO~\citep{RLOOAhmadianCGFKPUH24} on the MMK12 training dataset~\citep{meng2025mmeurekaexploringfrontiersmultimodal}, which contains 15,616 multimodal math problems with verified answers. 
All RL algorithms use QwenVL2.5-7B as the base model, trained under the same hyperparameters: rollout and training batch sizes of 128, 8 sampled responses per prompt (temperature 1.0), a learning rate of $1{e}{-}6$, one PPO epoch, and five training episodes. No reference policy constraint is applied during training, and final performance is reported using the last checkpoint. 
In our system prompt, reasoning steps and final answers are explicitly marked using \texttt{<think>} and \texttt{<answer>} tags, respectively (see Appendix~\ref{apx: about experiment}).


\textbf{Benchmarks, model baselines and Overall metric. }
We evaluate all algorithms on six widely used benchmarks: MathVista (testmini)~\citep{lu2024mathvistaevaluatingmathematicalreasoning}, MathVerse (testmini)~\citep{zhang2024mathversedoesmultimodalllm}, MathVision (test)~\citep{wang2024measuringmultimodalmathematicalreasoning}, OlympiadBench (EN-OE split)~\citep{he2024olympiadbenchchallengingbenchmarkpromoting}, WeMath~\citep{qiao2024we} and MMK12~\citep{meng2025mmeurekaexploringfrontiersmultimodal}. 
MathVista covers visual QA, logic, algebra, and geometry; MathVerse focuses on mathematically grounded visual understanding; and MathVision extends to abstract visual reasoning. OlympiadBench targets graduate-level competition problems, while WeMath enables fine-grained diagnostic analysis via hierarchically annotated tasks. MMK12 provides 500 multiple-choice questions per subject across math, physics, chemistry, and biology for cross-domain performance evaluation.

We also include several multimodal models as baselines. We evaluate open-source models of comparable model size, trained with various strategies, including QwenVL2.5-7B~\citep{bai2023qwenvlversatilevisionlanguagemodel}, InternVL2.5-8B~\citep{chen2024expanding}, InternVL2.5-MPO-8B~\citep{wang2024enhancingreasoningabilitymultimodal}, R1-OneVision~\citep{yang2025r1onevision}, OpenVLThinker~\citep{deng2025openvlthinker}, and MM-Eureka~\citep{meng2025mmeurekaexploringfrontiersmultimodal}, which collectively represent the average performance of this model size across the evaluated benchmarks. 
We further evaluate the leading closed-source models such as GPT-4o~\citep{hurst2024gpt} and OpenAI-o1~\citep{openai2024o1} to represent the most outstanding performance that the current state-of-the-art model can achieve on these benchmarks. 


To capture overall model performance across $N$ benchmarks, we define an \textit{Overall} metric by normalizing each score against a strong baseline, QwenVL2.5-7B:
\(\text{Overall}:= \frac{1}{N}\sum_{j=1}^N {X_{j}}/{X_{j}^\text{Qwen}}\), where $X_{j}$ and $X_{j}^\text{Qwen}$ are the model and baseline scores on benchmark $j$.

\subsection{Main results}


\begin{table}[h]
    \centering
    \caption{Performance comparison of various 7B/8B models and leading closed-source models. Top performer is in \textbf{bold} and second-best is \underline{underlined}  (excl. OpenAI-o1/GPT-4o). } 
    \label{tab:benchmark_comparison}
    \setlength{\tabcolsep}{2pt}
    \begin{tabular}{lccccccc}
        \toprule
        \textbf{Model} & {MathVista} & {MathVerse} & {MathVision} & {Olypamid} & {WeMath} & {MMK12} & {Overall} \\
        \midrule
        \multicolumn{8}{l}{\textbf{Leading models}} \\
        GPT-4o & 63.8 & 50.2 & 30.4 & 35.0 & 68.8  & 49.9 &1.16\\
        OpenAI-o1 & 73.9 & 57.0 & 60.3 & 68.0 & 98.7 & 73.9 &1.83  \\
        \midrule
        \multicolumn{8}{l}{\textbf{Similar-size models}} \\
        InternVL2.5-8B & 64.4 & 39.5 & 19.7 & 12.3 & 53.5 & 45.6 &0.81  \\
        QwenVL2.5-7B & 68.2 & 47.9 & 25.4 & 20.2 & 62.1 & 53.6 &1.00 \\
        InternVL2.5-MPO-8B & 68.9 & 35.5 & 21.5 & \, 7.8 & 53.5 & 34.5 &0.75 \\
        R1-Onevision (7B) & 64.1 & 47.1 & 23.5 & 17.3 & 61.8  & 39.8 &0.91 \\
        OpenVLThinker (7B) & 70.2 & 47.9 & 25.3 & 20.1 & 64.3 & 60.6 &1.03 \\
        MM-Eureka (7B) & 73.0 & 50.3 & \underline{26.9} & 20.1 & 66.1 & 64.5 &1.07 \\
        \midrule
        \multicolumn{8}{l}{\textbf{Different RL algorithms on QwenVL2.5-7B}} \\
        RLOO & 68.6 & 48.3 & 23.0 & 19.5 & 65.8 & 61.3 &1.01 \\
        REINFORCE++ & 63.9 & 45.5 & 18.2 & 17.8 & 66.7 & 64.3 &0.96 \\
        GRPO  & 70.3 & \textbf{51.4} & {25.9} & 18.5 & 67.4 & 65.1 &1.06 \\
        \rowcolor{RoyalBlue!5}\emph{\textbf{CPGD}} (clip-filter-like) & \underline{73.4} & \textbf{51.4} & {25.9} & \textbf{21.5} & \textbf{70.2} & \textbf{67.3} &\underline{1.10} \\
        \rowcolor{RoyalBlue!5}\emph{\textbf{CPGD}} (STD weight) & \textbf{74.0} & \underline{50.6} & \textbf{28.3} & \underline{21.4} & \underline{68.3} & \underline{65.3} &\textbf{1.11} \\
        \bottomrule
    \end{tabular}
\end{table}

Table~\ref{tab:benchmark_comparison} presents a comprehensive comparison across multiple multimodal mathematical benchmarks. Closed-source models GPT-4o and OpenAI-o1 demonstrate strong performance across all tasks, with o1 achieving the highest scores overall, notably excelling on MathVision (60.3), Olypamid (68.0) and WeMath (98.7), establishing the current performance upper bound. 
Among similar-size open models, MM-Eureka shows competitive results. MM-Eureka achieves strong results on MathVista (73.0), MathVision (26.9) and a strong result on MMK12 (64.5). However, our proposed CPGD consistently outperforms all similar-size baselines, achieving top or near-leading scores across all benchmarks, reflecting the effectiveness of our proposed RL algorithm. 

We further analyze different RL algorithms under the same setting as ours, including the base model, the training dataset, and the hyperparameters. Among baseline methods, GRPO outperforms RLOO and REINFORCE++ on most benchmarks, particularly on MathVerse (51.4) and MathVision (25.9). However, our proposed CPGD method significantly outperforms all baselines, achieving the best performance. Both variants of CPGD (using either clip-filter-like weights or STD-based weights) yield over a +10\% improvement in overall performance compared to the base model QwenVL2.5-7B. Notably, CPGD (STD weight) achieves a +21.8\% gain on the in-domain benchmark MMK12, and further demonstrates strong generalization with +8.5\% and +11.4\% improvements on the out-of-distribution benchmarks MathVista and MathVision, respectively. These results demonstrate that CPGD serves as a strong and robust alternative for RL in LM training. 


\subsection{Ablation study}\label{sec: ablation study}


\begin{table}[h]
    \centering
    \caption{Results of ablation studies. Top performer is in \textbf{bold} and second-best is \underline{underlined}. }  
    \label{tab:ablation study}
    \setlength{\tabcolsep}{1.6pt}
    \begin{tabular}{lcccccccc}
        \toprule
        \textbf{Model} & {MathVista} & {MathVerse} & {MathVision} & {Olypamid} & {WeMath} & {MMK12} & {Overall} \\
        \midrule
        \rowcolor{RoyalBlue!5}\emph{\textbf{CPGD}} (STD weight)  & \textbf{74.0} & {50.6} & \textbf{28.3} & \underline{21.4} & {68.3} & {65.3} &\textbf{1.11} \\
        \midrule
        \multicolumn{8}{l}{\textbf{Ablation study on the components (using STD weight)}} \\
        PG & 67.8 & 42.0 & 22.5 & 8.0 & 58.6 & 65.9 &0.89 \\
        PGD & 64.2 & 41.1 & 20.8 & 7.5 & 58.3 & \textbf{67.3} &0.86 \\
        CPG & 72.7 & \textbf{52.3} & \underline{27.6} & 20.8 & \textbf{70.7} & 66.2 &\textbf{1.11} \\
        \midrule
        \multicolumn{8}{l}{\textbf{Ablation study on the weighting factor}} \\
        unprocessed rewards & 69.1 & 40.2 & 21.8 & 3.5 & 59.7 & \underline{67.2} &0.85 \\
        equal weight & 73.1 & 51.1& {27.2} & 20.8 & 67.9 & 65.8 &1.09 \\
        clip-filter-like weight & \underline{73.4} & \underline{51.4} & 25.9 & \textbf{21.5} & \underline{70.2} & \textbf{67.3} &\underline{1.10} \\
        \midrule
        \multicolumn{8}{l}{\textbf{Ablation study on the reference constraint (using STD weight)}} \\
        w/ reference constraint  & 71.8 & 50.0 & 21.0 & 21.2 & 69.8 & 65.8 &1.05 \\
        \bottomrule
    \end{tabular}
\end{table}

\textbf{Component ablation.}
We conduct ablation on key components of our method by comparing variants: PG (basic policy gradient), PGD (PG + policy drift), CPG (PG + clip mechanism), and CPGD. Results show that the clip mechanism plays the most critical role, as seen by the performance drop from CPG/CPGD to PG/PGD across nearly all benchmarks. This aligns with our observation in Section~\ref{sec: training collapse} that clipping mitigates the response length collapse issue, which otherwise can impair test-time computation and reasoning capabilities. In contrast, adding policy drift has a relatively smaller effect. This is because CPGD’s objective lacks a potentially unstable importance-sampling ratio and already benefits from proximal updates via clipping, making policy drift mainly serve as a safeguard against excessive ratio deviation. 

\textbf{Weighting factor ablation.}
We further ablate different weighting strategies. We additionally include a baseline that uses raw \textit{unprocessed rewards} as advantages, which results in significant performance degradation. This confirms that subtracting the group mean is crucial for stable and effective learning. This approach prevents over-penalization of all responses in the failure cases, which may otherwise trigger a \textit{squeezing effect}~\citep{ren2025learning}, where the $\operatorname{Softmax}$ output head unintentionally reallocates probability mass to unexpected tokens, resulting in undesirable behaviors. 
Both clip-filter-like weight and STD weight outperform equal weighting, which we attribute to their ability to assign greater emphasis to samples with non-zero advantages. This targeted weighting encourages the model to focus more on informative training signals, thereby contributing to the improved performance. 

\textbf{Reference constraint ablation.}
Removing the reference constraint consistently improves performance, which echoes findings from recent studies~\citep{yu2025dapoopensourcellmreinforcement, liu2025drgrpounderstandingr1zeroliketrainingcritical, OpenReasonerZero2025}, suggesting that such constraints may overly restrict policy improvement, and thus hinder overall optimization.



\section{Discussion}

\subsection{Importance sampling}


Importance sampling is a valuable technique for correcting the sampling distribution when the learned policy and the behavior policy differ significantly, thereby improving sample efficiency. While we omit the importance-sampling ratio to reduce variance, we \textbf{do not} suggest discarding it entirely.
In fact, we use a single PPO epoch during training, a widely recommended default~\citep{hu2025reinforce++, meng2025mmeurekaexploringfrontiersmultimodal}. In our view, importance sampling can be omitted with one epoch but should be reintroduced when using more: 
\[
    A_{\omega}^\text{CPGD}(\mathbf{x}, \mathbf{y}) \gets \operatorname{clip}_{1-\epsilon}^{1+\epsilon}\Big(\frac{\operatorname{sg}(\pi_{\theta^{(m-1)}}(y_i|\mathbf{x}, \mathbf{y}_{<i}))}{\pi_{\theta_{old}}(y_i|\mathbf{x}, \mathbf{y}_{<i}))}\Big) A_{\omega}^\text{CPGD}(\mathbf{x}, \mathbf{y}),\quad m=1,\dots,M, 
\]
where $\pi_{\theta^{(m)}}$ denotes the updated policy after the $m$-th PPO epoch, and $\pi_{\theta^{(0)}} = \pi_{\theta_{old}}$, and thus the final updated policy is $\theta_{new} = \theta^{(M)}$ after total $M$ epochs. Here, the truncated importance sampling weight is applied to correct the off-policy distribution. 
Notably, we use $\theta^{(m)}$ rather than the real-time $\theta$ to avoid instability caused by frequent updates within a single PPO epoch. This also ensures consistency with our proposed method. However, maintaining $\pi_{\theta^{(m)}}$ may incur additional cost, which we leave for future work to optimize.





\subsection{Forward KL divergence vs. reverse KL divergence}



Our policy drift adopts the \textit{forward KL divergence} $D_\text{KL}(\pi_{\theta_{old}}, \pi_{\theta}|\mathbf{x})$ instead of the \textit{reverse KL divergence} $D_\text{KL}(\pi_{\theta}, \pi_{\theta_{old}}|\mathbf{x})$. 
While forward KL has been explored before~\citep{PPO}, it is considered less effective than PPO-clip. In contrast, reverse KL is more commonly used in theory because it is closely related to mirror descent and has strong convergence guarantees~\citep{geist2019theory, shani2020adaptive}. 

Although these two KL forms are different in how they are calculated, they often lead to similar results in practice~\citep{hsu2020revisitingdesignchoicesproximal}. This is because both are used to control policy updates. In fact, the difference between their gradients turns out to be small when the policy ratio is small, which is usually the case during training as shown in Figure~\ref{fig: training collapse}:
\[
    \nabla_{\theta} D_\text{KL}(\pi_{\theta}, \pi_{\theta_{old}}|\mathbf{x}) - \nabla_{\theta} D_\text{KL}(\pi_{\theta_{old}}, \pi_{\theta}|\mathbf{x})
    {\approx}
    \mathbb{E}_{\mathbf{y}\sim\pi_{\theta_{old}}(\cdot|\mathbf{x})} \Big[
    \frac12\Big(\frac{\pi_{\theta}(\mathbf{y}|\mathbf{x})}{\pi_{\theta_{old}}(\mathbf{y}|\mathbf{x})} - 1 \Big)^2
    \nabla_{\theta} \ln \pi_{\theta}(\mathbf{y}|\mathbf{x}) 
    \Big].
\]
This approximation holds because ${x\ln x \approx x - 1  + \frac12(x-1)^2}$ when $x$ is close to 1. 
Despite their similarity, we prefer forward KL for two main reasons: (1) It avoids importance sampling, which reverse KL requires; and (2) It can be cleanly split into per-token terms (see Equation~\ref{eq: k3}), which is not possible with reverse KL due to the importance weights.

\subsection{Exploitation vs. exploration}



Recent work~\citep{yue2025doesreinforcementlearningreally} claims that the performance ceiling of a model is determined by its base model, casting a pessimistic view on the role of RL. While we do not fully agree or disagree, we offer a more nuanced view: \textit{the exploration capability is largely determined by the base model}.

In RL training for LMs, the set of possible responses is constrained by what the base model can generate. RL helps it pick the best ones, boosting metrics like ${\text{Maj}@K}$. In other words, pretraining and SFT shape what the model can explore, while RL enhances the model’s exploitation ability.

This work mainly aims to improve RL stability, but advancing LM reasoning capability requires improving both RL and earlier stages like SFT to expand the model’s exploration range. Encouraging active exploration may be key to unlocking further improvements in model performance.


\section{Conclusion}

We identify a critical source of instability in existing RL methods for LMs: the use of asymmetric clipping on importance-sampling ratios, which can result in training collapse. To address this, we propose \textit{CPGD}, a principled alternative that avoids direct dependence on policy ratios while enforcing proximal updates through the clip mechanism and policy drift. CPGD further incorporates a stable KL estimator and a weighted advantage strategy to improve learning robustness. Theoretically grounded and empirically validated, CPGD demonstrates superior stability and performance across multimodal math benchmarks, offering a strong and stable RL solution for training LMs.





{
\small
\bibliographystyle{unsrt}
\bibliography{neurips_2025}
}


\newpage

\appendix

\section*{Appendix}
\section{Proofs}\label{apx: proof}

\subsection{Proof for Proposition~\ref{prop: ratio out of the range}}

\begin{proposition}\label{apx: prop: ratio out of the range}
    Let $\theta_0$ be a parameter such that the importance-sampling ratio satisfies $|\frac{\pi_{\theta_0}(\mathbf{y}|\mathbf{x})}{\pi_{\theta_{old}}(\mathbf{y}|\mathbf{x})} - 1|= \epsilon$. Consider updating $\theta_0$ using either (i) the PPO-clip objective, resulting in parameter $\theta_1^{\text{PPO}}$, or (ii) the CPGD objective with $\alpha = 0$, yielding parameter $\theta_1^{\text{CPG}}$. Then, there exists a constant $\eta_{\max} > 0$ such that for any learning rate $\eta \in (0, \eta_{\max})$, the following inequality holds: 
    \begin{align*}
        \Bigg|\frac{\pi_{\theta_1^\text{PPO}}(\mathbf{y}|\mathbf{x})}{\pi_{\theta_{old}}(\mathbf{y}|\mathbf{x})} - 1\Bigg|>
        \Bigg|\frac{\pi_{\theta_1^\text{CPG}}(\mathbf{y}|\mathbf{x})}{\pi_{\theta_{old}}(\mathbf{y}|\mathbf{x})} - 1\Bigg|>
        \Bigg|\frac{\pi_{\theta_0}(\mathbf{y}|\mathbf{x})}{\pi_{\theta_{old}}(\mathbf{y}|\mathbf{x})} - 1\Bigg|= \epsilon.
    \end{align*}
    After one update step, both PPO and CPG increase the importance-sampling ratio deviation from the old policy, but PPO does so more aggressively than CPG. 
\end{proposition}
\begin{proof}
    Consider $f(\eta) = \frac{\pi_{\theta_1^\text{CPG}}(\mathbf{y}|\mathbf{x})}{\pi_{\theta_{old}}(\mathbf{y}|\mathbf{x})}$, where $\theta_1^\text{CPG} = \theta_0 + \eta \nabla_\theta \mathcal{\hat L}_\text{CPG}(\mathbf{x}, \mathbf{y}; \theta_0)$ is the single gradient ascent step on the empirical CPGD objective (Equation~\ref{eq: CPGD theory}) without the policy drift term. The gradient of the objective takes the form:
    \begin{align*}
        \nabla_\theta \mathcal{\hat L}_\text{CPG}(\mathbf{x}, \mathbf{y}; \theta) = A^\text{CPGD}(\mathbf{x}, \mathbf{y})\nabla_\theta \ln \pi_{\theta}(\mathbf{y} | \mathbf{x} ).
    \end{align*}
    Thus, for the case where $\frac{ \pi_{\theta_0}(\mathbf{y}|\mathbf{x})}{\pi_{\theta_{old}}(\mathbf{y}|\mathbf{x})} = 1+\epsilon$ and $A^\text{CPGD}(\mathbf{x}, \mathbf{y})>0$, the directional derivative of $f$ at $\eta=0$ satisfies:
    \begin{align*}
        f^\prime(0) = \langle{
        \frac{\nabla_\theta \pi_{\theta_0}(\mathbf{y}|\mathbf{x})}{\pi_{\theta_{old}}(\mathbf{y}|\mathbf{x})}, \nabla_\theta \mathcal{\hat L}_\text{CPG}(\mathbf{x}; \theta_0)
        }\rangle > 0. 
    \end{align*}
    Hence, there exists a constant $\eta_1>0$ such that for any $\eta\in(0, \eta_{1})$, we have $f(\eta) > f(0)$. 
    Similarly, when $\frac{ \pi_{\theta_0}(\mathbf{y}|\mathbf{x})}{\pi_{\theta_{old}}(\mathbf{y}|\mathbf{x})} = 1-\epsilon$ and $A^\text{CPGD}(\mathbf{x}, \mathbf{y})<0$, there exists $\eta_2>0$ such that $f(\eta) < f(0)$ for any $\eta\in(0, \eta_{2})$.
    
    Therefore, for any $0< \eta < \min(\eta_1, \eta_2)$, the following holds:
    \begin{align}\label{apx: prop1: result 1}
        |\frac{\pi_{\theta_1^\text{CPG}}(\mathbf{y}|\mathbf{x})}{\pi_{\theta_{old}}(\mathbf{y}|\mathbf{x})} - 1|>
        |\frac{\pi_{\theta_0}(\mathbf{y}|\mathbf{x})}{\pi_{\theta_{old}}(\mathbf{y}|\mathbf{x})} - 1|= \epsilon.
    \end{align}

    Next, define $g(\eta) = \frac{\pi_{\theta_1^\text{CPG}}(\mathbf{y}|\mathbf{x})}{\pi_{\theta_{old}}(\mathbf{y}|\mathbf{x})} - \frac{\pi_{\theta_1^\text{PPO}}(\mathbf{y}|\mathbf{x})}{\pi_{\theta_{old}}(\mathbf{y}|\mathbf{x})}$, where $\theta_1^\text{PPO} = \theta_0 + \eta \nabla_\theta \mathcal{\hat L}_\text{PPO}(\mathbf{x}, \mathbf{y}; \theta_0)$ and
    \begin{align*}
        \nabla_\theta \mathcal{\hat L}_\text{PPO}(\mathbf{x}, \mathbf{y}; \theta) = A^\text{CPGD}(\mathbf{x}, \mathbf{y}) \frac{\nabla_\theta\pi_{\theta}(\mathbf{y}|\mathbf{x})}{\pi_{\theta_{old}}(\mathbf{y}|\mathbf{x})}.
    \end{align*}
    For the case where $\frac{ \pi_{\theta_0}(\mathbf{y}|\mathbf{x})}{\pi_{\theta_{old}}(\mathbf{y}|\mathbf{x})} = 1+\epsilon$ and $A^\text{CPGD}(\mathbf{x}, \mathbf{y})>0$, we have: 
    \begin{align*}
        g^\prime(0) = 
        \Big\langle{
        \frac{\nabla_\theta \pi_{\theta_0}(\mathbf{y}|\mathbf{x})}{\pi_{\theta_{old}}(\mathbf{y}|\mathbf{x})}, A^\text{CPGD}(\mathbf{x}, \mathbf{y}) \cdot (1-\frac{\pi_{\theta}(\mathbf{y}|\mathbf{x})}{\pi_{\theta_{old}}(\mathbf{y}|\mathbf{x})}) \cdot \nabla_\theta \ln\pi_{\theta}(\mathbf{y}|\mathbf{x})
        }\Big\rangle < 0.
    \end{align*}
    Hence, there exists a constant $\eta_3>0$ such that $g(\eta) < g(0)$ for any $\eta\in(0, \eta_{3})$. 
    Similarly, for the case where $\frac{ \pi_{\theta_0}(\mathbf{y}|\mathbf{x})}{\pi_{\theta_{old}}(\mathbf{y}|\mathbf{x})} = 1-\epsilon$ and $A^\text{CPGD}(\mathbf{x}, \mathbf{y})<0$, there exists a constant $\eta_4>0$ such that $g(\eta) > g(0)$ for any $\eta\in(0, \eta_{4})$. 
    
    Therefore, for any $0 < \eta < \min(\eta_3, \eta_4)$, we have
    \begin{align}\label{apx: prop1: result 2}
        |\frac{\pi_{\theta_1^\text{PPO}}(\mathbf{y}|\mathbf{x})}{\pi_{\theta_{old}}(\mathbf{y}|\mathbf{x})} - 1|>
        |\frac{\pi_{\theta_1^\text{CPG}}(\mathbf{y}|\mathbf{x})}{\pi_{\theta_{old}}(\mathbf{y}|\mathbf{x})} - 1|. 
    \end{align}

    Therefore, by letting $\eta_{\max} = \min(\eta_1, \eta_2, \eta_3, \eta_4)$, the proof is complete. 
\end{proof}

\subsection{Proof for Theorem~\ref{thrm: convergence of CPGD}}

\begin{theorem}\label{apx: thrm: convergence of CPGD}
    Let $\{\pi_{\theta_k}\}_{k=0}^\infty$ denote the sequence of policies generated by the CPGD update rule (Equation~\ref{eq: CPGD theory}). Then, the sequence ${\pi_{\theta_k}}$ converges. 
\end{theorem}
\begin{proof}
First, denote $\mathcal{L}_\text{CPGD}(\theta;\theta_{k}) = \mathbb{E}_{\mathbf{x}\sim\mathcal{D}}\big[  g(\theta;\theta_{k}, \mathbf{x}) \big]$, and rewrite $g$ as 
\begin{align*}
    g(\theta;\theta_{k}, \mathbf{x}) &= 
    \mathbb{E}_{\mathbf{y}\sim\pi_{\theta_k}(\cdot|\mathbf{x})}\Big[ \mathcal{R}_o(\mathbf{x}, \mathbf{y}) \ln\frac{\pi_\theta(\mathbf{y}|\mathbf{x})}{\pi_{\theta_{k}}(\mathbf{y}|\mathbf{x})} \Big]
    -\alpha D_\text{KL}(\pi_{\theta_k}, \pi_{\theta}|\mathbf{x}) 
    \\
    -&\mathbb{E}_{\mathbf{y}\sim\pi_{\theta_{k}}(\cdot|\mathbf{x})}\Big[
    \text{ReLU}\Big(\big[
    \ln\frac{\pi_\theta(\mathbf{y}|\mathbf{x})}{\pi_{\theta_{k}}(\mathbf{y}|\mathbf{x})}   - 
    \operatorname{clip}\big( \ln\frac{\pi_\theta(\mathbf{y}|\mathbf{x})}{\pi_{\theta_{k}}(\mathbf{y}|\mathbf{x})}  , \ln(1-\epsilon), \ln(1+\epsilon) \big)\big]\mathcal{R}_o(\mathbf{x}, \mathbf{y})
    \Big)
    \Big].
\end{align*}
Here, we omit the baseline $\mathbb{E}_{\mathbf{y}\sim\pi_{\theta_{k}}(\cdot|\mathbf{x})}[\mathcal{R}_o(\mathbf{x}, \mathbf{y})]$. 
Then, denoting $\theta_{k+1}$ the point such that $\mathcal{L}_\text{CPGD}(\theta_{k+1};\theta_{k})\ge \mathcal{L}_\text{CPGD}(\theta_k;\theta_{k})$, we obtain
\begin{align*}
    &\mathbb{E}_{\mathbf{y}\sim\pi_{\theta_{k+1}}(\cdot|\mathbf{x})}\Big[ \mathcal{R}_o(\mathbf{x}, \mathbf{y}) \Big] - \mathbb{E}_{\mathbf{y}\sim\pi_{\theta_{k}}(\cdot|\mathbf{x})}\Big[ \mathcal{R}_o(\mathbf{x}, \mathbf{y}) \Big]
    \\=&  
    \mathbb{E}_{\mathbf{y}\sim\pi_{\theta_{k}}(\cdot|\mathbf{x})}\Big[ \big(\frac{\pi_{\theta_{k+1}}(\mathbf{y}|\mathbf{x})}{\pi_{\theta_{k}}(\mathbf{y}|\mathbf{x})} - 1\big)\mathcal{R}_o(\mathbf{x}, \mathbf{y}) \Big]
    \\\ge&
    \mathbb{E}_{\mathbf{y}\sim\pi_{\theta_{k}}(\cdot|\mathbf{x})}\Big[ \ln \frac{\pi_{\theta_{k+1}}(\mathbf{y}|\mathbf{x})}{\pi_{\theta_{k}}(\mathbf{y}|\mathbf{x})} \cdot \mathcal{R}_o(\mathbf{x}, \mathbf{y}) \Big]
    \\=&
    g(\theta_{k+1};\theta_{k}, \mathbf{x}) - g(\theta_{k};\theta_{k}, \mathbf{x})
    +\alpha D_\text{KL}(\pi_{\theta_k}, \pi_{\theta_{k+1}}|\mathbf{x})
    \\&+\mathbb{E}_{\mathbf{y}\sim\pi_{\theta_{k}}(\cdot|\mathbf{x})}\Big[\text{ReLU}\Big(\big[
    \ln\frac{\pi_{\theta_{k+1}}(\mathbf{y}|\mathbf{x})}{\pi_{\theta_{k}}(\mathbf{y}|\mathbf{x})}   - 
    \operatorname{clip}\big( \ln\frac{\pi_{\theta_{k+1}}(\mathbf{y}|\mathbf{x})}{\pi_{\theta_{k}}(\mathbf{y}|\mathbf{x})}  , \ln(1-\epsilon), \ln(1+\epsilon) \big)\big]\mathcal{R}_o(\mathbf{x}, \mathbf{y})
    \Big)
    \Big].
\end{align*}
Denoting the overall expected return by $\eta(\pi_{\theta}) = \mathbb{E}_{\mathbf{x}\sim\mathcal{D}, \mathbf{y}\sim\pi_{\theta}(\cdot|\mathbf{x})}\Big[ \mathcal{R}_o(\mathbf{x}, \mathbf{y}) \Big]$, we integrate over $\mathbf{x}$ to conclude
\begin{align*}
    \eta(\pi_{\theta_{k+1}}) - \eta(\pi_{\theta_{k}}) \ge \alpha \mathbb{E}_{\mathbf{x}\sim\mathcal{D}}\Big[D_\text{KL}(\pi_{\theta_k}, \pi_{\theta_{k+1}}|\mathbf{x})\Big]\overset{\text{Pinsker inequality}}{\ge} \frac{\alpha}{2}\Vert{\pi_{\theta_{k+1}} - \pi_{\theta_{k}}}\Vert_1^2. 
\end{align*}
Because $\eta(\pi_{\theta_{k}})$ is bounded, there exists a $\eta_*$ such that $\lim_{k\rightarrow\infty} \eta(\pi_{\theta_{k}})=\eta_*$. 
Thus, taking the limit of $k$ on both sides of the following equation, 
\begin{align*}
    0\le \frac{\alpha}{2}\Vert{\pi_{\theta_{k+1}} - \pi_{\theta_{k}}}\Vert_1^2 \le \eta(\pi_{\theta_{k+1}}) - \eta(\pi_{\theta_{k}}),
\end{align*}
we can obtain $\lim_{k\rightarrow\infty}\Vert{\pi_{\theta_{k+1}} - \pi_{\theta_{k}}}\Vert_1 = 0$. 
Since the parameter space $\Theta$ is compact, the sequence $\{\pi_{\theta_k}\} $ converges to some limit point $\pi_{\theta_*}$.

\end{proof}


\section{Prompt setting}\label{apx: about experiment}

\begin{table}[ht]
    \centering
    \caption{Prompt setting. }
    \label{prompt}
    \begin{tabular}{p{12cm}}
        \toprule
        \textbf{SYSTEM:} 
        Solve the question. The user asks a question, and you solves it. You first thinks about the reasoning process in the mind and then provides the user with the answer. The answer is in latex format and wrapped in \$...\$. The final answer must be wrapped using the \textbackslash boxed\{\} command. The reasoning process and answer are enclosed within \textless think\textgreater \textless /think\textgreater \;and \textless answer\textgreater \textless /answer\textgreater \;tags, respectively, i.e., \textless think\textgreater Since $1+1=2$, so the answer is $2$. \textless /think\textgreater \textless answer\textgreater The answer is \$\textbackslash boxed\{2\}\$ \textless /answer\textgreater, which means the final answer assistant's output should start with \textless answer\textgreater \;and end with \textless /answer\textgreater.
        \\
        \textbf{USER:} \textless image\textgreater \textcolor{red}{\{\{question\}\}} \\
        \bottomrule
    \end{tabular}
\end{table}

We follow the prompt format from DeepSeek-R1, where reasoning steps and final answers are explicitly marked using \texttt{<think>} and \texttt{<answer>} tags, respectively. The full prompt template is provided in Table~\ref{prompt}.

\section{Limitations}\label{apx: limitation}

While this work introduces a stable and effective RL method for LMs training, it has several limitations: (1) For the weighted advantage component, we conducted only preliminary experiments and did not thoroughly explore the impact of different weighting factors. Our results suggest that using non-uniform weights yields better performance than trivial equal weighting, but further investigation is needed. (2) Our study focuses on on-policy training; we leave off-policy settings—where importance sampling is typically required—for future work. Ensuring training stability in the presence of importance sampling remains an open question. (3) All experiments were conducted on standard academic-scale models (7B parameters). We did not evaluate our method on larger models (e.g., 100B+), which would require significant computational resources.

\end{document}